\documentclass[preprint]{l4dc2025}

\usepackage{xpatch,bm,mathtools}
\usepackage{mleftright}
\usepackage{xcolor}
\usepackage{pgfplotstable}
\usepackage{filecontents}
\usepackage{multicol}
\usepackage{multirow}
\usepackage{siunitx}
\usepackage{lipsum}
\usepackage{cite}
\usepackage{subcaption}
\usepackage{paralist}
\usepackage[capitalize]{cleveref}
\usepackage[labelfont=bf, font={footnotesize}]{caption}
\usepackage{tikz}
\usepackage{wrapfig} 

\usepgfplotslibrary{external}
\usepgfplotslibrary{fillbetween}
\usepgfplotslibrary{statistics}
\usetikzlibrary{positioning}
\usetikzlibrary{calc}

\RestyleAlgo{ruled, linesnumbered}

\newtheorem{problem}{Problem}

\DeclareMathOperator*{\argmin}{arg\,min}

\newcommand{\bs}{\bm{s}}
\newcommand{\ba}{\bm{a}}
\newcommand{\btau}{\bm{\tau}}

\newcommand{\R}{\mathbb{R}}

\newcommand{\tp}{^\mathsf{T}}

\newlength{\fwidth}
\newlength{\fheight}

\title[Diffusion Predictive Control with Constraints]{Diffusion Predictive Control with Constraints}
\usepackage{times}

\author{%
 \Name{Ralf R\"omer} 
 \Email{ralf.roemer@tum.de} 
 \AND
 \Name{Alexander {von Rohr}} 
 \Email{alex.von.rohr@tum.de} 
 \AND
 \Name{Angela P. Schoellig} 
 \Email{angela.schoellig@tum.de}\\
 \addr Technical University of Munich, Germany; TUM School of Computation, Information and Technology, 
 Learning Systems and Robotics Lab; Munich Institute of Robotics and Machine Intelligence (MIRMI)
}

\begin{document}
\begin{tikzpicture}[remember picture, overlay]
    \draw[white, fill=white] 
        ([xshift=3cm,yshift=-1.4cm]current page.north west) rectangle 
        ([xshift=11cm,yshift=-1.8cm]current page.north west);
\end{tikzpicture}
\vspace{-1.1cm}

\maketitle

\begin{abstract}%
Diffusion models have become popular for policy learning in robotics due to their ability to capture high-dimensional and multimodal distributions.
However, diffusion policies are stochastic and typically trained offline, limiting their ability to handle unseen and dynamic conditions where novel constraints not represented in the training data must be satisfied.
To overcome this limitation, we propose diffusion predictive control with constraints (DPCC), an
algorithm for diffusion-based control with explicit state and action constraints that can deviate from those in the training data.
DPCC incorporates model-based projections into the denoising process of a trained trajectory diffusion model and uses constraint tightening to account for model mismatch. This allows us to generate constraint-satisfying, dynamically feasible, and goal-reaching trajectories for predictive control.
We show through simulations of a robot manipulator that DPCC outperforms existing methods in satisfying novel test-time constraints while
maintaining performance on the learned control task. 
\end{abstract}

\begin{keywords}%
Diffusion Policies, Imitation Learning, Predictive Control, Robotics %
\end{keywords}

\section{Introduction}
Recent advances in using diffusion models~\citep{sohl2015deep, ho2020denoising} for control and robotics have demonstrated their potential in applications such as robot manipulation \citep{chi2023diffusion} and locomotion \citep{huang2024diffuseloco}. 
These works highlight how diffusion models excel at learning policies directly from diverse demonstrations, capturing multimodal behavior, and handling high-dimensional state and action spaces. 
However, as diffusion models are trained offline to generate samples via an iterative stochastic denoising process, they don't account for hard constraints, especially if these are not present in the training data. 
This limits their applicability in controlling robots under changing operation conditions and in real-world dynamic environments where novel and potentially time-varying constraints, such as avoiding obstacles, must be satisfied.

In contrast, the ability to modify a plan according to unforeseen circumstances is a hallmark of model-based control and planning approaches, such as model predictive control (MPC)~\citep{rawlings2017model}. 
These methods leverage a model of the system dynamics to generate feasible trajectories that satisfy constraints despite external disturbances by incorporating feedback. 
However, MPC requires a cost function, which is often difficult to formulate for complex robotics tasks.

These limitations motivate combining the expressiveness and flexibility of diffusion-based control policies with the ability of MPC to satisfy hard constraints.
Different strategies have been proposed to incorporate constraints into diffusion models, either during training~\citep{bastek2024physics}, inference~\citep{carvalho2023motion, christopher2024projected}, or both~\citep{ajay2023conditional} (see \cref{sec:related_work} for a more detailed discussion).
In this work, we propose Diffusion Predictive Control with Constraints (DPCC), an algorithm that extends diffusion policies to operate under unseen state and action constraints. DPCC achieves this by integrating repeated model-based projections into the trajectory denoising process. In addition, we adopt a constraint-tightening mechanism to account for errors in the dynamics model.
At each timestep, DPCC generates a batch of predicted trajectories that are dynamically feasible, constraint-satisfying, and perform the learned task, and applies the first action from a selected trajectory.
In summary, our main contributions are the following:
\begin{compactitem}
    \item We show that generating goal-reaching trajectories that are guaranteed to satisfy constraints can be achieved by incorporating model-based projections into the backward diffusion process.
    \item We propose additional constraint tightening to account for model errors and a selection mechanism for the trajectories generated by the diffusion model to improve task performance.
    \item We evaluate our approach in simulations of a robotic manipulator and demonstrate its superior performance in satisfying novel constraints while still reliably solving the task.
\end{compactitem}
These contributions collectively advance the state of the art in diffusion-based control policies, enabling their deployment in safety-critical environments by satisfying novel constraints.

\vspace{-4pt}
\section{Related Work}\label{sec:related_work}
\vspace{-2pt}
\textbf{Diffusion-Based Control}:
Diffusion models have recently been applied to various decision-making tasks such as imitation learning~(IL)~\citep{pearce2023imitating, chi2023diffusion, chen2023playfusion}, offline reinforcement learning~\citep{janner22planning, ajay2023conditional} and motion planning~\citep{carvalho2023motion, power2023sampling}. 
While some works generate only one action per timestep~\citep{pearce2023imitating, reuss2023goal}, most approaches adopt a receding horizon control strategy. 
This can be done by directly predicting state-action trajectories, either using a single~\citep{janner22planning} or two separate diffusion models~\citep{zhou2024diffusion}, or by predicting state (or high-level action) trajectories and using a separate controller~\citep{ajay2023conditional, chi2023diffusion}.
In this regard, we follow Diffuser~\citep{janner22planning} as this approach allows us to ensure dynamic feasibility. 
\\
%
\textbf{Diffusion Models with Constraints}:
Many generative modeling tasks require generating samples that are not only from the same distribution as the training data but also adhere to certain constraints.
If the constraints are always the same, a residual loss can be added to the training objective~\citep{bastek2024physics}.
A more flexible approach is to sample from a conditional distribution, where the conditioning variable represents a parameterization of the constraints.
Classifier-free guidance~\citep{ho2022classifier} trains additional diffusion models for each condition and can be used to encourage satisfaction of constraints seen in the training data or novel combinations of those constraints~\citep{ajay2023conditional}, but requires more labeled data.
It has also been proposed to formulate constraints via cost functions and add their gradients to the backward diffusion process~\citep{carvalho2023motion, kondo2024cgd}, which is conceptually similar to classifier guidance~\citep{dhariwal2021diffusion}.
However, training loss modification and model conditioning can only encourage but not guarantee constraint satisfaction of the generated samples.
Post-processing methods impose constraints on the generated samples by modifying them after the last denoising step, usually by solving an optimization problem~\citep{giannone2023aligning, power2023sampling, maze2023diffusion}.
As the optimization problem does not consider the unknown data likelihood, post-processing may result in samples that significantly deviate from the data distribution.
To address this problem, the integration of projections into the denoising process has recently been investigated. However, these approaches either disregard the system dynamics and deviations between the learned and the actual distribution~\citep{romer2024safe}, resulting in frequent constraint violations in closed-loop operation or are too computationally expensive for sequential decision-making~\citep{christopher2024projected}.


\section{Problem Statement}
We consider a dynamical system with state~${\bs_t \in \mathcal{S}}$ and action~${\ba_t \in \mathcal{A}}$ at timestep~${t}$ that is governed by the discrete-time dynamics
\begin{align}
    \label{eq:prob_dynamics}
    \bs_{t+1} = \bm{f}(\bs_t,\ba_t) + \bm{w}_t,
\end{align}
%
where $\bm{f}$ is known, and $\bm{w}_t$ is an unknown disturbance (or model mismatch) bounded by~$\|\bm{w}_t\|_2 \leq \gamma$ for all~$t$.
We aim to control the system~\eqref{eq:prob_dynamics} such that a goal~$\bm{g} \in \mathcal{G}$ is reached, which is indicated by a binary indicator function ${\phi:\mathcal{S} \times \mathcal{G} \rightarrow \{0,1\}}$.
For this, we assume the availability of a dataset
\begin{align}
    \label{eq:prob_dataset}
    \mathcal{D} = \left\{\bm{\tau}_\text{e}^{(n)} = \big(\bs_{0}^{(n)}, \ba_{0}^{(n)}, \dots, \bs_{T_n}^{(n)}, \ba_{T_n}^{(n)}, \bm{g}^{(n)}\big)\right\}_{n=1}^N
\end{align}
containing~$N$ demonstrations of system~\eqref{eq:prob_dynamics} performing the desired task, i.e., ${\phi
\big(\bs_{T_n}^{(n)}, \bm{g}^{(n)}\big) = 1}$ for all~${n \in \mathbb{I}_1^N=\{1,\dots,N\}}$.
The dataset has been collected by an unknown stochastic expert policy~$\bm{\pi}_\text{e}$, i.e., ${\ba_t \sim \bm{\pi}_\text{e}(\cdot|\bs_t,\bm{g})}$.
We consider the demonstrations~\eqref{eq:prob_dataset} to be multimodal, i.e., they contain multiple distinct ways to reach the goal~\citep{jia2024towards, urain2024deep}. 
Therefore, we aim to use a diffusion model to learn a stochastic policy~$\bm{\pi}$ from the data~\eqref{eq:prob_dataset} via imitation learning. 
In addition, our objective is to satisfy novel and potentially time-varying state and action constraints
\begin{align}
    \label{eq:prob_constraints}
    \bs_t \in \mathcal{S}_t \subseteq \mathcal{S}, \quad \ba_t \in \mathcal{A}_t \subseteq \mathcal{A}, \qquad \forall t,
\end{align}
at test-time, where we assume the sets~$\mathcal{S}_t$ and~$\mathcal{A}_t$ to be closed for all~$t$. 
We refer to the constraints~\eqref{eq:prob_constraints} as \textit{novel} because we do not assume that they are satisfied by some or all of the demonstrations in the training dataset~\eqref{eq:prob_dataset}.
Such novel constraints can arise, for example, when deploying a learned robot policy in an environment with moving obstacles or when system specifications, such as torque or velocity limits, are different at test time than during data collection. 
\begin{problem}
    Given a demonstration dataset~\eqref{eq:prob_dataset}, how can we obtain a diffusion policy~$\bm{\pi}$ that can control the system~\eqref{eq:prob_dynamics} to reach a desired goal \textit{and} satisfy the constraints~\eqref{eq:prob_constraints}?
\end{problem}

\section{Background on Trajectory Diffusion} \label{sec:background}

Diffusion models are generative models for learning an unknown target distribution~$q$ from samples~${\btau^0\sim q(\cdot)}$, which we consider to be trajectories~${\btau^0 = (\bs_{0:T}, \ba_{0:T})}$ of system~\eqref{eq:prob_dynamics}.
The main idea is to gradually transform the data into noise and learn a reverse process to reconstruct the data from pure noise~\citep{sohl2015deep}.
Denoising diffusion probabilistic models~(DDPM) \citep{ho2020denoising} introduce latent variables~${\btau^1,\dots,\btau^K}$ and construct a forward diffusion process
\begin{align}
    \label{eq:pre_forward_diffusion_process}
    q\big(\btau^k|\btau^{k-1},k\big) = \mathcal{N}\big(\sqrt{1-\beta_k} \btau^{k-1}, \beta_k \bm{I}\big),
\end{align}
where~${k=1,\dots,K}$ is the diffusion time step and the values~${\beta_{1:k} \in (0,1)^k}$ are determined by a noise schedule. Since the transition dynamics~\eqref{eq:pre_forward_diffusion_process} are Gaussian, we can compute marginals in closed form as~$q\big(\btau^k|\btau^0,k\big) = \mathcal{N}\left(\sqrt{\bar{\alpha}_k} \btau^0, (1-\bar{\alpha}_k) \bm{I}\right)$,
%
%
where~${\alpha_k = 1-\beta_k}$, ${\bar{\alpha}_k = \prod_{i=1}^k \alpha_k}$.
The noise schedule and the number of diffusion steps~$K$ are chosen such that~${q\big(\btau^K| \btau^0,K\big) \approx \mathcal{N}(\bm{0}, \bm{I})}$, i.e., the forward process gradually transforms the trajectories data into Gaussian noise.
This process is reversed by the learnable backward diffusion (or denoising) process
\begin{align}
    \label{eq:pre_backward_diffusion_process}
    p_{\bm{\theta}}\big(\btau^{k-1}|\btau^k,k\big) = \mathcal{N}\big(\bm{\mu}_{\bm{\theta}}\big(\btau^k, k\big), \bm{\Sigma}_{\bm{\theta}}\big(\btau^k, k\big)\big), \qquad p\big(\btau^K\big) = \mathcal{N}(\bm{0},\bm{I}),
\end{align}
where~$\bm{\mu}_{\bm{\theta}}$ and~$\bm{\Sigma}_{\bm{\theta}}$ can be parameterized by neural networks. The training objective is to match the joint distributions in the forward and backward process, i.e., ${q\big(\btau^{0:K}\big)=q\big(\btau^0\big)\prod_{k=1}^K q\big(\btau^{k}|\btau^{k-1},k\big)}$ and~${p_{\bm{\theta}}\big(\btau^{0:K}\big) = p\big(\btau^K\big)\prod_{k=1}^Kp_{\bm{\theta}}\big(\btau^{k-1}|\btau^k,k\big)}$, by maximizing the evidence-lower bound~(ELBO) \citep{ho2020denoising}.
The variance is often set to~${\bm{\Sigma}_{\bm{\theta}}\left(\btau^k, k\right) = \sigma_k^2 \bm{I}}$, where~${\sigma_k^2 = \beta_k \frac{1-\bar{\alpha}_{k-1}}{1-\bar{\alpha}_k}}$, and the mean~$\bm{\mu}_{\bm{\theta}}$ is learned indirectly by learning to predict the noise added to~$\btau^0 \sim q(\cdot)$ via the surrogate loss function~${\mathcal{L}(\bm{\theta}) = \mathbb{E}_{k\sim \mathrm{Unif}(1,K), \btau^0\sim q(\cdot), \bm{\epsilon} \sim \mathcal{N}(\bm{0},\bm{I})} \left[ \big\|\bm{\epsilon} - \bm{\epsilon}_{\bm{\theta}}\big(\sqrt{\bar{\alpha}_k} \btau^0 + \sqrt{1-\bar{\alpha}_k}\bm{\epsilon}, k\big) \big\|_2 \right]}$,
where ${\bm{\epsilon}_{\bm{\theta}}\big(\btau^k,k\big) = \frac{\sqrt{1-\bar{\alpha}_k}}{\beta_k} \left(\btau^k - \sqrt{\alpha_k}\bm{\mu}_{\bm{\theta}}\big(\btau^k,k\big)\right)}$. 
After training with this loss, we can generate trajectories from the learned distribution~${\btau^0 \sim p_{\bm{\theta}}(\cdot)}$ by iterating through the backward diffusion process~\eqref{eq:pre_backward_diffusion_process}.
It is also possible to learn and sample from a conditional distribution~$\btau^0 \sim p_{\bm{\theta}}(\cdot|\bm{c})$, where~$\bm{c}$ is some context, via methods such as inpainting~\citep{sohl2015deep}, classifier guidance~\citep{dhariwal2021diffusion} or classifier-free guidance~\citep{ho2022classifier}.


\section{Methodology} \label{sec:meth}

In this section, we present the DPCC algorithm. We explain the use of diffusion models for receding horizon control and show how to incorporate novel constraints in the backward diffusion process via model-based projections. Lastly, we account for model errors using constraint tightening and introduce two trajectory selection criteria. The DPCC method is summarized in Algorithm~\ref{alg:receding_horizon_control}.

\subsection{Diffusion-Based Receding Horizon Control} \label{sec:diffusion_rhc}
We address the problem of learning a control policy from an offline dataset~\eqref{eq:prob_dataset} via conditional generative modeling~\citep{janner22planning, ajay2023conditional}. 
For this, we consider state-action trajectories~${\btau = (\bs_{t:t+H}, \ba_{t:t+H})}$ of horizon length~${H+1}$ of system~\eqref{eq:prob_dynamics}. 
The expert policy~$\bm{\pi}_\text{e}$ induces a conditional trajectory distribution~$q(\bm{\tau}|\bm{c})$, where~${\bm{c}=(\bs_t, \bm{g})}$, which is generally unknown.
Utilizing the samples from~$q(\bm{\tau}|\bm{c})$ in~\eqref{eq:prob_dataset}, we train a diffusion model to learn a trajectory distribution
\begin{align}
    \label{eq:meth_learned_traj_dist}
    p_{\bm{\theta}}(\btau| \bm{c}) \approx q(\btau|\bm{c}),
\end{align}
as described in~\cref{sec:background}, where the learned backward diffusion process is given by
\begin{align}
    \label{eq:meth_backward_diffusion_process}
    p_{\bm{\theta}}\big(\btau^{k-1}|\btau^k, k, \bm{c}\big) = \mathcal{N}\big(\bm{\mu}_{\bm{\theta}}\big(\btau^k, k, \bm{c}\big), \sigma_k^2 \bm{I}\big), \qquad p\big(\btau^K\big) = \mathcal{N}(\bm{0},\bm{I}).
\end{align}
As the learned distribution~\eqref{eq:meth_learned_traj_dist} implicitly encodes information about the dynamics~\eqref{eq:prob_dynamics} and how to solve the task, we can use it for receding horizon control: At time~$t$, given the current state~$\bs_t$ and the goal~$\bm{g}$, we sample a future trajectory~${\btau_{t:t+H|t}=(\bs_{t:t+H|t}, \ba_{t:t+H|t})}$ 
from~\eqref{eq:meth_learned_traj_dist} and apply the first action~$\ba_{t|t}$. Here~$\ba_{t+i|t}$ denotes the action prediction for time~$t+i$ generated at time~$t$~\citep{borrelli2017predictive}.
This process is repeated until the goal is reached, i.e., ~${\phi(\bs_t, \bm{g}) = 1}$.
A major limitation of this approach is that it cannot take into account constraints of the form~\eqref{eq:prob_constraints}. In the following, we will therefore discuss how to incorporate such constraints into diffusion-based receding-horizon control.

\subsection{Constraint- and Model-Based Trajectory Diffusion}
Diffusion predictive control as discussed in~\cref{sec:diffusion_rhc} generates predicted trajectories 
from the learned distribution~\eqref{eq:meth_learned_traj_dist}.
In that way, the approach implicitly encourages satisfaction of constraints that were already present in the demonstration dataset~\eqref{eq:prob_dataset}, such as state and action bounds.
However, the controlled system may still violate these constraints because the denoising process~\eqref{eq:meth_backward_diffusion_process} is stochastic and the true trajectory distribution can generally only be approximated. 
Moreover, our main goal is to satisfy novel constraints of the form~\eqref{eq:prob_constraints}.
As a first step, we need to ensure that these constraints are satisfied by the open-loop trajectories~$\btau_{t:t+H|t}$ predicted at timestep~$t$, i.e., 
\begin{align}
\begin{split}
    \label{eq:meth_trajectory_constraint}
    \btau_{t:t+H|t} \in \mathcal{Z}
    = \big\{\btau = (\bs_{t:t+H}, \ba_{t:t+H})|\,\bs_{t'} \in {\mathcal{S}}_{t'}, \ba_{t'} \in \mathcal{A}_{t'}, \; \forall t' \in \mathbb{I}_{t}^H\big\}.
\end{split}
\end{align}
One way to enforce~\eqref{eq:meth_trajectory_constraint} is to perform a projection of the denoised trajectory~${\btau^0 = \btau_{t:t+H|t}}$ into the set~$\mathcal{Z}$ as~${\Pi_{\mathcal{Z}}\big(\btau^0\big) = \argmin_{\bm{\tilde{\tau}} \in \mathcal{Z}} \big\|\btau^0 - \bm{\tilde{\tau}}\big\|_2}$.
However, since this projection only takes into account the state and action constraints~\eqref{eq:prob_constraints}, using the projected trajectory for receding-horizon control has two drawbacks: The resulting trajectory may not be dynamically feasible anymore, and it may not be suitable for achieving the goal~$\bm{g}$. 
We can mitigate the first problem by taking into account the dynamics~\eqref{eq:prob_dynamics} and applying a model-based constraint set projection to~$\btau^0$, which is defined by
\begin{subequations}
\label{eq:meth_projection_with_prior}
\begin{align}
    \Pi_{\mathcal{Z}_{\bm{f}}}\big(\btau\big) = \argmin_{\bm{\tilde{\tau}} = (\bs_{t:t+H|t}, \ba_{t:t+H|t}) \in \mathcal{Z}} \; &\|\btau - \bm{\tilde{\tau}}\|_2^2 \\
    \text{s.t.} \qquad \quad &\bm{s}_{t'+1|t} = \bm{f}\big(\bs_{t'|t},\ba_{t'|t}\big), \qquad \forall t' \in \mathbb{I}_t^H,
\end{align}
\end{subequations}
where~$\mathcal{Z}_{\bm{f}} = \big\{\btau = (\bs_{t:t+H}, \ba_{t:t+H})|\,\btau \in \mathcal{Z}, \bs_{t+1} = \bm{f}(\bs_t,\ba_t),\, \forall t \in \mathbb{I}_t^H\big\}$ is assumed to be non-empty, and the projection cost is denoted by~${c_{\mathcal{Z}_{\bm{f}}}(\btau) = \|\btau - \Pi_{\mathcal{Z}_{\bm{f}}}(\btau)\|_2^2}$.
Since the projection is goal-independent, it may render the trajectory less useful for completing the task. 
As the training dataset~\eqref{eq:prob_dataset} consists of dynamically feasible and goal-reaching trajectories, these two properties are implicitly encoded in the learned trajectory distribution~\eqref{eq:meth_learned_traj_dist}, which is defined by the backward process~\eqref{eq:meth_backward_diffusion_process}.
Thus, we aim to modify~\eqref{eq:meth_backward_diffusion_process} only as much as necessary to guarantee constraint satisfaction.

We approach this problem via control as inference~\citep{toussaint2009robot} and introduce a binary variable~$\mathcal{O}\in\{0,1\}$ that is related to the feasibility of a trajectory~$\btau$, i.e., whether~$\btau \in \mathcal{Z}_{\bm{f}}$.
We can then formulate our objective as sampling trajectories from the conditional distribution
\begin{align}
    \label{eq:meth_trajectory_distribution_feasible}
    p_{\bm{\theta}}(\btau | \mathcal{O}=1) \propto p_{\bm{\theta}}(\btau) p(\mathcal{O}=1|\btau)
\end{align}
instead of the original learned distribution~\eqref{eq:meth_learned_traj_dist}, where we have omitted the conditioning on the context~$\bm{c}$ for brevity. If the likelihood~$p(\mathcal{O}|\btau)$ is defined as
\begin{align}
    \label{eq:meth_feasibility_likelihood_ideal}
    p(\mathcal{O}=1|\btau) = \begin{cases}
        1, \text{ if } \btau \in \mathcal{Z}_{\bm{f}} \\
        0, \text{ otherwise},
    \end{cases}
\end{align}
sampling from~\eqref{eq:meth_trajectory_distribution_feasible} is guaranteed to yield feasible trajectories.
In principle, this could be performed through rejection sampling from the learned distribution~\eqref{eq:meth_learned_traj_dist}, but this becomes too computationally inefficient if samples~$\btau \sim p_{\bm{\theta}}(\cdot)$ are unlikely to lie within~$\mathcal{Z}_{\bm{f}}$.
Instead, we can sample from~\eqref{eq:meth_trajectory_distribution_feasible} more efficiently if the likelihood~$p(\mathcal{O}|\btau)$ takes a different form than~\eqref{eq:meth_feasibility_likelihood_ideal}.
\begin{theorem} \label{the:condition_on_constraints}
    Let~$\mathcal{Z}_{\bm{f}}$ be a closed convex set, $\sigma_k > 0$, $\forall k \in \mathbb{I}_{1}^K$,
    and let the feasibility likelihood be defined by~${p\big(\mathcal{O} = 1 | \btau, k\big) \propto \exp{\left(-\frac{1}{2\sigma_k^2}d\big(\btau, \mathcal{Z}_{\bm{f}}\big)^2\right)}}$, where ${d(\btau, \mathcal{Z}_{\bm{f}}) = \min_{\bm{\tilde{\tau}} \in \mathcal{Z}_{\bm{f}}} \|\bm{\tilde{\tau}} - \btau\|_2}$ is the distancee between~$\btau$ and~$\mathcal{Z}_{\bm{f}}$. 
    Then, we can approximately sample from~\eqref{eq:meth_trajectory_distribution_feasible} via the modified denoising process
    \begin{align}
        \label{eq:meth_backward_diffusion_conditional}
        p_{\bm{\theta}}\big(\btau^{k-1}|\btau^k, k, \mathcal{O}=1\big) 
        \approx 
        \mathcal{N}\big(
        \Pi_{\mathcal{Z}_{\bm{f}}} \big(\bm{\mu}_{\bm{\theta}}^k\big), \sigma_k^2 \bm{I}\big), \qquad p\big(\btau^K, \mathcal{O}\big) = \mathcal{N}(\bm{0},\bm{I}),
    \end{align}
    where the learned mean~${\bm{\mu}_{\bm{\theta}}^k = \bm{\mu}_{\bm{\theta}}\big(\bm{\tau}^k,k\big)}$ is the same as in~\eqref{eq:meth_backward_diffusion_process}, and~$\Pi_{\mathcal{Z}_{\bm{f}}}$ is defined in~\eqref{eq:meth_projection_with_prior}.
\end{theorem}

\begin{proof}
    Incorporating the conditioning on~$\mathcal{O}$ into the Markovian backward diffusion process gives
    \begin{align}
        \label{eq:meth_backward_diffusion_conditional_def}
        p_{\bm{\theta}}\big(\btau^{k-1}|\btau^k, \mathcal{O},k\big) 
        &\propto
        p_{\bm{\theta}}\big(\btau^{k-1}|\btau^k,k\big) p\big(\mathcal{O}|\btau^{k-1},k\big),
    \end{align}
    where~$p_{\bm{\theta}}\big(\btau^{k-1}|\btau^k,k\big) = \mathcal{N}\big(\bm{\mu}_{\bm{\theta}}^k, \sigma_k^2 \bm{I}\big)$.
    As the likelihood~$p(\mathcal{O}| \btau,k)$ is smooth with respect to~$\btau$ by definition, its logarithm at~$\bm{\tau} = \btau^{k-1}$ can be approximated using a first-order Taylor expansion around~$\bm{\mu}_{\bm{\theta}}^k$, which gives ${\log{p(\mathcal{O}| \btau^{k-1},k)} \approx \log{p\big(\mathcal{O}| \bm{\mu}_{\bm{\theta}}^k,k\big)} + \big(\btau^{k-1} - \bm{\mu}_{\bm{\theta}}^k\big)\tp \bm{v}(\mathcal{O})}$,
    where ${\bm{v}(\mathcal{O}) = \nabla_{\btau} \log p(\mathcal{O}|\btau,k)|_{\btau = \bm{\mu}_{\bm{\theta}}^k}}$. 
    This approximation allows us to rewrite~\eqref{eq:meth_backward_diffusion_conditional_def} as
    \begin{align}
        \label{eq:meth_backward_diffusion_conditional_proof}
        p_{\bm{\theta}}\big(\btau^{k-1}|\btau^k, k, \mathcal{O}\big) 
        &\approx
        \mathcal{N}\big(\bm{\mu}_{\bm{\theta}}^k + \sigma_k^2 \bm{v}(\mathcal{O}), \sigma_k^2 \bm{I} \big),
    \end{align}
    as shown in the derivation of classifier guidance~\citep{dhariwal2021diffusion}.
    Since~$\mathcal{Z}_{\bm{f}}$ is closed and convex by assumption, 
    there exists a unique projection~${\bm{z} = \Pi_{\mathcal{Z}_{\bm{f}}}\big(\bm{\mu}_{\bm{\theta}}^k\big)}$ for each~$\bm{\mu}_{\bm{\theta}}^k$; see~\citep{bazaraa2006nonlinear}, Theorem 2.4.1.
    Thus, we can write~$\bm{v}(\mathcal{O} = 1)$ as
    \begin{align}
        \label{eq:meth_gradient_feasibility_reformulated}
        \bm{v}(\mathcal{O} = 1)
        = -\frac{1}{\sigma_k^2} d\big(\bm{\mu}_{\bm{\theta}}^k, \mathcal{Z}_{\bm{f}}\big) \nabla_{\btau} d(\btau, \mathcal{Z}_{\bm{f}})|_{\btau = \bm{\mu}_{\bm{\theta}}^k}
        = \frac{1}{\sigma_k^2} \big(\bm{z} - \bm{\mu}_{\bm{\theta}}^k\big).
    \end{align}
    Inserting~\eqref{eq:meth_gradient_feasibility_reformulated} into~\eqref{eq:meth_backward_diffusion_conditional_proof} and replacing~$\bm{z}$ by~$\Pi_{\mathcal{Z}_{\bm{f}}}\big(\bm{\mu}_{\bm{\theta}}^k\big)$ yields~\eqref{eq:meth_backward_diffusion_conditional}, which concludes the proof.
\end{proof}
\begin{algorithm}[tb!]
\small
\hrulefill
\vspace{-2pt}
\caption{DPCC: Diffusion Predictive Control with Constraints.} \vspace{-6pt}
\hrulefill

\textbf{Input: } Diffusion model~$\bm{\epsilon}_{\bm{\theta}}$, goal~$\bm{g}$, dynamics~$\bm{f}$, state constraints~$\mathcal{S}_{0,1,\dots}$, action constraints~$\mathcal{A}_{0,1,\dots}$. \par
Set $t=0$. \par
Compute tightened state constraints~$\tilde{\mathcal{S}}_{0,1,\dots}$ via~\eqref{eq:meth_constraint_tightening}. \par
\While{goal~$\bm{g}$ not reached}{
Get current state~$\bm{s}_t$ and set~$\bm{c}=(\bs_t,\bm{g})$. \par
Sample a trajectory batch from noise: $\bm{\tau}_{t:t+H|t}^{K,1:B} \sim \mathcal{N}(\bm{0}, \bm{I})$. \par
\For{$k = K,\dots,1$}{
Trajectory denoising step:
$\bm{\tilde{\tau}}_{t:t+H|t}^{k-1,1:B} \sim 
\mathcal{N}\big(\bm{\mu}_{\bm{\theta}}\big(\btau_{t:t+H|t}^{k,1:B}, k, \bm{c}\big), \sigma_k^2 \bm{I}\big)$. \par 
Model-based projection into the feasible set: $\bm{\tau}_{t:t+H|t}^{k-1,1:B} = \Pi_{\tilde{\mathcal{Z}}_{\bm{f}}}\big(\bm{\tilde{\tau}}_{t:t+H|t}^{k-1,1:B}\big)$. \par 
}

Select a trajectory~$\btau^* = \bm{\tau}_{t:t+H|t}^{0,i}$ from~$\bm{\tau}_{t:t+H|t}^{0,1:B}$ via temporal consistency or projection costs~(\cref{sec:trajectory_selection}). \par
Apply the first action~$\ba_{t|t}$ in~$\btau^*$. \par
Set~$t \leftarrow t + 1$.
}
\vspace{-6pt}
\hrulefill
\label{alg:receding_horizon_control}
\end{algorithm}

\cref{the:condition_on_constraints} assumes that~$\mathcal{Z}_{\bm{f}}$ is convex. This is true, for example, if the constraint sets in~\eqref{eq:prob_constraints} are convex and the dynamics~\eqref{eq:prob_dynamics} are linear.
Moreover, with the likelihood definition in~\cref{the:condition_on_constraints}, $p(\mathcal{O}=1|\btau,k)>0$ for~$\btau \notin \mathcal{Z}_{\bm{f}}$. Consequently, sampling from~\eqref{eq:meth_trajectory_distribution_feasible} via~\eqref{eq:meth_backward_diffusion_conditional} is not strictly guaranteed to yield samples~${\btau^0 = \btau_{t:t+H|t} \in \mathcal{Z}_{\bm{f}}}$, i.e., trajectories are not guaranteed to satisfy~\eqref{eq:prob_constraints}.

Nonetheless, \cref{the:condition_on_constraints} provides theoretical justification to address constraint satisfaction via iterative projections in the denoising process.
To ensure that~$\btau^0 \in \mathcal{Z}_{\bm{f}}$ for any variance schedule~$\sigma_{1:K}$, we slightly modify~\eqref{eq:meth_backward_diffusion_conditional} and apply the projection \textit{after} adding the noise. This yields the model-informed denoising step with deterministic constraint satisfaction
\begin{align}
    \label{eq:meth_backward_diffusion_projected}
    \btau^{k-1} = \Pi_{\mathcal{Z}_{\bm{f}}}\big(\bm{\mu}_{\bm{\theta}}\big(\btau^k, k, \bm{c}\big) + \sigma_k \bm{\epsilon}_k \big), \qquad \bm{\epsilon}_k \sim \mathcal{N}(\bm{0}, \bm{I}).
\end{align}
which we use in DPCC. Here, we have included the context~$\bm{c}$ again.
By using~\eqref{eq:meth_backward_diffusion_projected}, We denote the samples projected distribution as~$\btau^0 \sim p_{\bm{\theta}}(\cdot|\bm{c},\mathcal{Z}_{\bm{f}})$.

\subsection{Constraint Tightening} \label{sec:constraing_tightening}
The dynamics model~$\bm{f}$ used in the projection~\eqref{eq:meth_projection_with_prior} is only an approximation of the true system~\eqref{eq:prob_dynamics}, which is subject to a model mismatch~$\bm{w}_t$. 
Hence, feasibility of the predicted trajectory, i.e., $\btau_{t:t+H|t} \in \mathcal{Z}_{\bm{f}}$, does not guarantee that the actual future states~$\bs_{t+1}, \bs_{t+2}, \dots$ will satisfy the constraints~\eqref{eq:prob_constraints}. 
We take this into account by tightening the constraints for the predicted states.

\begin{theorem} \label{the:tracking_error}
Let~${\bs_0 \in \mathcal{S}_0}$, ${\bm{g} \in \mathcal{G}}$, $\mathcal{B}_\gamma$ denote the $\ell_2$-norm ball of radius $\gamma$ and $\ominus$ the Minkowski set difference, and define the tightened constraint sets for all~$t$ as
\begin{align}
    \label{eq:meth_constraint_tightening}
    \tilde{\mathcal{S}}_{t+1} = \mathcal{S}_{t+1} \ominus \mathcal{B}_\gamma.
\end{align}
Then, if at each timestep~$t = 0,1,\dots$, we 
sample $\btau_{t:t+H|t}=(\bs_{t:t+H|t}, \ba_{t:t+H|t}) \sim p_{\bm{\theta}}\big(\cdot|\bm{c}, \tilde{\mathcal{Z}}_{\bm{f}}\big)$, where $\tilde{\mathcal{Z}}_{\bm{f}} = \big\{\btau = (\bs_{t:t+H}, \ba_{t:t+H})|\,\bs_{t'} \in {\tilde{\mathcal{S}}}_{t'}, \ba_{t'} \in \mathcal{A}_{t'}, \; \forall t' \in \mathbb{I}_{t}^H\big\}$, and apply the action~$\bm{a}_{t|t}$ to system~\eqref{eq:prob_dynamics}, all future states satisfy the constraints~\eqref{eq:prob_constraints}, i.e., $\bs_t \in \mathcal{S}_t$, $\forall t \in \{1,2,\dots\}$.
\end{theorem}
\begin{proof}
Let~$\bs_t \in \mathcal{S}_t$. Due to the definition of~$p_{\bm{\theta}}\big(\btau|\bm{c}, \tilde{\mathcal{Z}}_{\bm{f}}\big)$ via~\eqref{eq:meth_backward_diffusion_projected}, sampling ${\btau_{t:t+H|t} \sim p_{\bm{\theta}}\big(\btau|\bm{c}, \tilde{\mathcal{Z}}_{\bm{f}}\big)}$ implies~${\btau_{t:t+H|t} \in \tilde{\mathcal{Z}}_{\bm{f}}}$. 
Consequently, the predicted next state satisfies both ${\bs_{t+1|t} \in \tilde{\mathcal{S}}_{t+1}}$ and ${\bm{f}(\bs_t,\ba_{t|t}) = \bs_{t+1|t}}$.
Inserting the latter into the dynamics~\eqref{eq:prob_dynamics} gives~$\bs_{t+1} = \bm{f}(\bs_t,\ba_{t|t}) + \bm{w}_t = \bs_{t+1|t} + \bm{w}_t$. 
The model mismatch is bounded by~$\|\bm{w}_t\|_2 \leq \gamma$ by assumption, so we can write $\bs_{t+1} \in \tilde{\mathcal{S}}_{t+1} \oplus \mathcal{B}_{\gamma} = (\mathcal{S}_{t+1} \ominus \mathcal{B}_\gamma) \oplus \mathcal{B}_{\gamma} \subseteq \mathcal{S}_{t+1}$. As~$\bs_0 \in \mathcal{S}_0$, the result follows by induction.
\end{proof}
%


\subsection{Trajectory Selection} \label{sec:trajectory_selection}
By using a generative diffusion model for predictive control, we can generate not just one trajectory at teach timestep, but a batch of~$B$ candidate trajectories denoted by~$\btau_{t:t+H|t}^{0,1:B}$.
Many existing works~\citep{chi2023diffusion, ajay2023conditional} apply the actions from a trajectory randomly selected from the batch.
However, this does not take into account that the candidate trajectories may be diverse and not equally well suited for the task.
To improve control performance, we propose two different criteria for selecting a trajectory~$\btau_{t:t+H|t}^{0,i(t)}$ from the sampled batch:
\begin{compactitem}
    \item \textbf{\underline{T}emporal Consistency (DPCC-T)}: 
    Frequent replanning using a multimodal trajectory distribution can result in alternating between different behavior modes, which may impact task performance.
    We can avoid this by selecting the trajectory 
    deviating the least from the previous timestep via~$i(t) = \argmin_{j} \big\|\btau_{t:t+H-1|t}^{0,j} - \btau_{t:t+H-1|t-1}^{0,i(t-1)}\big\|_2$.
    %
    %
    \item \textbf{Cumulative Projection \underline{C}ost (DPCC-C)}:     
    We aim to preserve as much information as possible from the learned trajectory distribution~\eqref{eq:meth_learned_traj_dist} despite the modifications to the denoising process~\eqref{eq:meth_backward_diffusion_process}.
    This motivates selecting the trajectory 
    that has been modified the least by the projection operation in~\eqref{eq:meth_backward_diffusion_projected} during the whole denoising process, i.e., ${i(t) = \argmin_{j}\,\sum_{k=1}^K c_{\mathcal{Z}_{\bm{f}}}\big(\bm{\tilde{\tau}}_{t:t+H|t}^{k-1,j}\big)}$,
    %
    %
    where $\tilde{\cdot}$~denotes the trajectories before applying the projection; see~\cref{alg:receding_horizon_control}. 
\end{compactitem}
\section{Evaluation and Discussion}
Our simulation experiments primarily aim at answering the following questions:
\begin{compactitem}
    \item \textbf{Q1}: Can our proposed DPCC algorithm satisfy novel constraints and still solve the learned task?
    \item \textbf{Q2}: How does the proposed constraint-informed trajectory denoising method~\eqref{eq:meth_backward_diffusion_projected} perform compared to existing approaches for incorporating constraints into diffusion models?
    \item \textbf{Q3}: How important is the accuracy of the dynamics model used in the projections~\eqref{eq:meth_projection_with_prior}?
\end{compactitem}

\subsection{Setup}
We conduct our experiments\footnote{Our code is available at \href{https://github.com/ralfroemer99/dpcc}{\color{black}{github.com/ralfroemer99/dpcc}}.} 
in the \texttt{Avoiding} simulation environment~\citep{jia2024towards} shown in~\cref{fig:avoiding}~(a). 
The task is for a robot manipulator to reach a line (green) with its end-effector without colliding with one of the six obstacles (red). 
The state~${\bs_t\in \R^4}$ consists of the current and desired end-effector positions in the 2D~plane, and the action~${\ba_t \in \R^2}$ contains the desired Cartesian velocities, which are sent to a low-level controller.
The training data set~$\mathcal{D}$ contains~$96$ demonstrations, $4$ for each of the~$24$ different ways of navigating around the six obstacles, resulting in a highly multimodal expert trajectory distribution; see~\cref{fig:avoiding}~(b).
We consider different formulations of novel state constraints, which are defined by circular and halfspace areas that the end-effector must not enter, as shown in~\cref{fig:avoiding}~(c). To ensure that the generated trajectories do not violate the control action bounds, we define~$\mathcal{A}_t = \mathcal{A}$, where~$\mathcal{A}$ is the smallest bounding box containing all actions from the demonstration dataset~\eqref{eq:prob_dataset}.
The action constraints are defined as~$\mathcal{A}_t = \mathcal{A}$, where~$\mathcal{A}$ is the smallest bounding box containing all actions from the demonstration dataset~\eqref{eq:prob_dataset}.
We do not assume knowledge of the dynamics of the low-level controller.
Instead, we approximate the system dynamics~\eqref{eq:prob_dynamics} by a simple Euler integration, i.e., ${\bs_{t+1} = \bm{s}_t + \big[\bm{a}_{t}\tp,\, \bm{a}_{t}\tp\big]\tp t_{\text{s}} + \bm{w}_t}$, where~$t_{\text{s}}$ is the sampling time, and the model mismatch~$\bm{w}_t$ accounts for the low-level controller and the numerical error of the Euler integration.
The maximum episode length is~$300$ timesteps.
We estimate the upper bound~$\gamma$ on the model mismatch based on 100 policy rollouts without constraints to be~${\gamma = 0.025}$.
\newlength{\vshift}
\setlength{\vshift}{0.2cm}

\begin{figure}[t]
    \centering
    \begin{subfigure}[b]
        \centering
        \includegraphics[height=27mm, trim={1.9cm 0 2.2cm 0},clip]{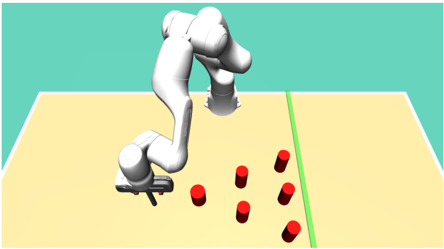}
        \label{fig:avoiding_env}
    \end{subfigure}
    \hfill
    \begin{subfigure}[b]
        \centering
        \includegraphics[height=27mm]{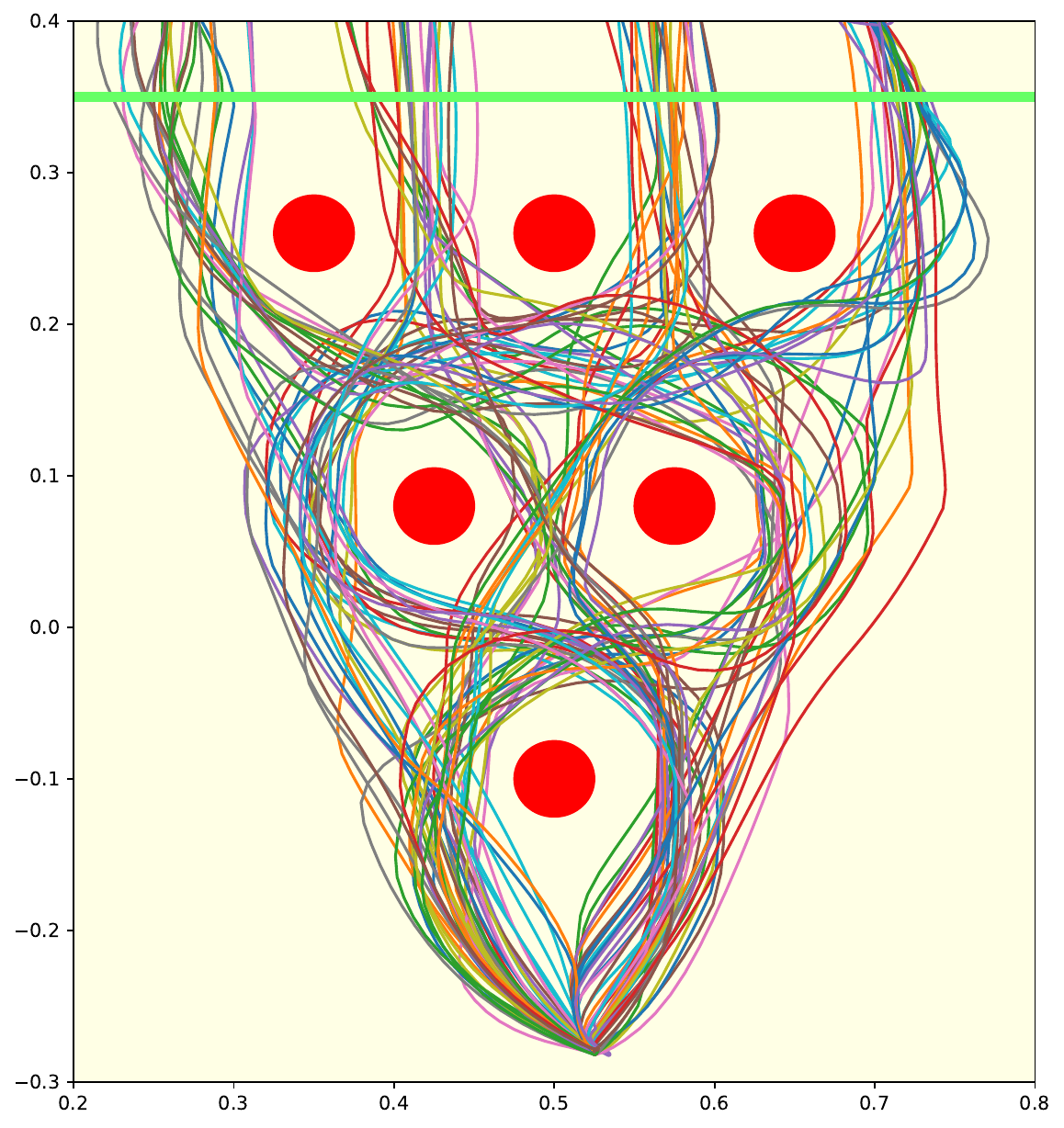}
        \label{fig:avoiding_data}
    \end{subfigure}
    \hfill
    \begin{subfigure}[b]
        \centering
        \includegraphics[height=27mm]{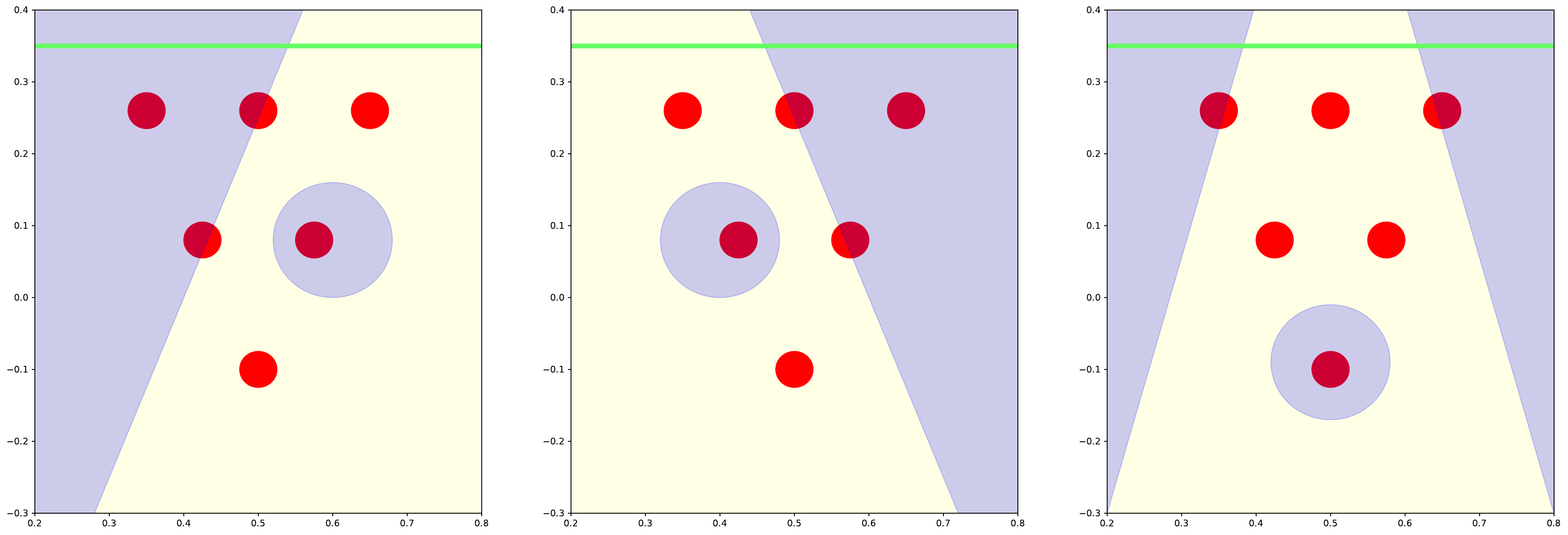}
        \label{fig:avoiding_constraints}
    \end{subfigure}
    \vspace{-3pt}
    \caption{Experiments: (a) Simulation environment, where the objective is to reach the green line with the end-effector without collisions. (b) Multimodal trajectory distribution in the training dataset. (c) Novel test-time constraints (blue).}
    \label{fig:avoiding}
\end{figure}


We use a 1D U-Net~\citep{janner22planning} as the diffusion model backbone~$\bm{\epsilon}_{\bm{\theta}}$, employ the cosine noise schedule~\citep{nichol2021improved} for~$\beta_{1:K}$ with~$K=20$ diffusion steps, and condition on the current state using inpainting. At each timestep, we sample a batch of~$B=4$ trajectories with horizon length~${H+1 = 8}$.
The state constraints render the feasible sets~$\tilde{\mathcal{Z}}_{\bm{f}}$ non-convex, and solve the resulting nonlinear optimization problems in the projections~\eqref{eq:meth_projection_with_prior} using an SLSQP solver~\citep{virtanen2020scipy, kraft1994algorithm}.
Computing an action with DPCC takes about~$80\,\si{ms}$ on a workstation with 64 GB RAM, an NVIDIA Geforce RTX 4090 GPU, and an Intel Core i7-12800HX CPU.

We compare DPCC against three baselines for satisfying constraints with diffusion policies:
\begin{compactitem}
    \item \textbf{Guidance}: The constraints are parameterized via cost functions, and their gradients are used to guide the denoising process towards the feasible set~\citep{carvalho2023motion, kondo2024cgd}. We conduct an ablation study to determine suitable weights for the gradient terms.
    \item \textbf{Post-Processing}: The trajectories are projected into the set~$\mathcal{Z}_{\bm{f}}$ only after the last denoising iteration~\citep{giannone2023aligning, power2023sampling}.
    \item \textbf{Model-Free}: The projection only takes into account the constraints~\eqref{eq:prob_constraints}, but not the system dynamics~\citep{romer2024safe}.
\end{compactitem}
None of these prior works use constraint tightening, but we evaluate their performance with and without our constraint tightening method to ensure the comparison is fair.

We use four evaluation metrics: The number of timesteps to reach the goal, the success rates of 1) reaching the goal and 2) satisfying the constraints and reaching the goal, and the number of timesteps in which the constraints are violated. 
All results are averaged over five training seeds and ten test seeds for each constraint set formulation.

\vspace{-3pt}
\subsection{Results}
\vspace{-1pt}
\pgfplotsset{compat=newest}
\begin{figure}[tb!]
\centering
\setlength{\fwidth}{5cm}
\setlength{\fheight}{2.1cm}
\definecolor{mycolor3}{RGB}{214,39,40}
\definecolor{mycolor2}{RGB}{44,160,44}
\definecolor{mycolor4}{RGB}{14,127,200}
\newlength{\barwidth}
\setlength{\barwidth}{\fwidth/9}
\newlength{\barshift}
\setlength{\barshift}{0.2cm}
\begin{tikzpicture}

\begin{axis}[
    width=\fwidth,
    height=\fheight,
    ticklabel style = {font=\footnotesize},
    label style = {font=\footnotesize},
    xtick style={draw=none}, 
    title={No constraint tightening},
    title style = {font=\footnotesize, yshift=-7},
    scale only axis,
    ybar,
    enlarge x limits=0.25,
    bar width=\barwidth, 
    ylabel={Success rate~$\uparrow$},
    ylabel style={yshift=-3pt},
    ytick={0, 0.5, 1},
    symbolic x coords={DPCC-R, DPCC-T, DPCC-C},
    xticklabel style={text width=0.3 * \fwidth, align=center, yshift=6pt},
    xtick=data,
    ymin=0, ymax=1.15,
    nodes near coords,
    nodes near coords style={font=\scriptsize, yshift=-2.5pt},
    legend style={at={(1,0)}, anchor=south east, legend cell align=left, legend columns=4, align=left,font=\scriptsize},
    ]
    \addplot[fill=mycolor2] coordinates {(DPCC-R, 0.97) (DPCC-T, 0.95) (DPCC-C, 0.96)};    
    \addplot[fill=mycolor3] coordinates {(DPCC-R, 0.41) (DPCC-T, 0.51) (DPCC-C, 0.49)};     
    \legend{Goal, Constraints \& goal}
\end{axis}

\begin{axis}[
    width=\fwidth,
    height=\fheight,
    xshift=\fwidth + 1cm,
    ticklabel style = {font=\footnotesize},
    label style = {font=\footnotesize},
    xtick style={draw=none}, 
    title={With constraint tightening},
    title style = {font=\footnotesize, yshift=-7},
    scale only axis,
    ybar,
    enlarge x limits=0.25,
    bar width=\barwidth, 
    ylabel style={yshift=-3pt},
    ytick={0, 0.5, 1},
    symbolic x coords={DPCC-R, DPCC-T, DPCC-C},
    xticklabel style={text width=0.3 * \fwidth, align=center, yshift=6pt},
    xtick=data,
    ymin=0, ymax=1.15,
    nodes near coords,
    nodes near coords style={font=\scriptsize, yshift=-2.5pt},
    legend style={at={(1,0)}, anchor=south east, legend cell align=left, legend columns=4, align=left,font=\scriptsize},
    ]
    \addplot[fill=mycolor2] coordinates {(DPCC-R, 1.00) (DPCC-T, 0.99) (DPCC-C, 1.00)};    
    \addplot[fill=mycolor3] coordinates {(DPCC-R, 0.96) (DPCC-T, 0.99) (DPCC-C, 0.98)};     
    \legend{Goal, Constraints \& goal}
\end{axis}

\begin{axis}[
    width=\fwidth,
    height=\fheight,
    yshift=-(\fheight + 0.5cm),
    ticklabel style = {font=\footnotesize},
    label style = {font=\footnotesize},
    xtick style={draw=none}, 
    title style = {font=\footnotesize, yshift=-7},
    scale only axis,
    ybar,
    enlarge x limits=0.25,
    bar width=\barwidth, 
    ylabel={Timesteps~$\downarrow$},
    ylabel style={yshift=-3pt},
    ytick={60, 70, 80, 90, 100},
    symbolic x coords={DPCC-R, DPCC-T, DPCC-C},
    xticklabel style={text width=0.3 * \fwidth, align=center, yshift=6pt},
    xtick=data,
    ymin=53, ymax=100,
    nodes near coords,
    nodes near coords style={font=\scriptsize, yshift=-2.5pt, color=black},
    ]
    \addplot+[
        fill=mycolor4, 
        draw=black,
        error bars/.cd,
        y dir=both, 
        y explicit, 
        error bar style={color=black}
    ] coordinates {
        (DPCC-R, 78.6) +- (0, 19.1) 
        (DPCC-T, 74.8) +- (0, 21.1)
        (DPCC-C, 76.1) +- (0, 21.5)
    };    
\end{axis}

\begin{axis}[
    width=\fwidth,
    height=\fheight,
    xshift=\fwidth + 1cm,
    yshift=-(\fheight + 0.5cm),
    ticklabel style = {font=\footnotesize},
    label style = {font=\footnotesize},
    xtick style={draw=none}, 
    title style = {font=\footnotesize, yshift=-7},
    scale only axis,
    ybar,
    enlarge x limits=0.25,
    bar width=\barwidth, 
    ylabel style={yshift=-3pt},
    ytick={60, 70, 80, 90, 100},
    symbolic x coords={DPCC-R, DPCC-T, DPCC-C},
    xticklabel style={text width=0.3 * \fwidth, align=center, yshift=6pt},
    xtick=data,
    ymin=53, ymax=100,
    nodes near coords,
    nodes near coords style={font=\scriptsize, yshift=-2.5pt, color=black},
    ]
    \addplot+[
        fill=mycolor4, 
        draw=black,
        error bars/.cd,
        y dir=both, 
        y explicit, 
        error bar style={color=black}
    ] coordinates {
        (DPCC-R, 73.0) +- (0, 10.6) 
        (DPCC-T, 66.2) +- (0, 8.6)
        (DPCC-C, 69.0) +- (0, 12.9)
    };    
\end{axis}

\end{tikzpicture}%
\vspace{-5pt}
\caption{Impact of our constraint tightening method and the trajectory selection criterion (DPCC-R: random, DPCC-T: temporal consistency, DPCC-C: cumulative projection cost) on the success rates 
and the number of timesteps needed.
}
\label{fig:selection_tightening_impact}
\vspace{-3pt}
\end{figure}

To answer \textbf{Q1}, we conduct an ablation study with regard to two key components of DPCC: The constraint-tightening method~(\cref{sec:constraing_tightening}) and the trajectory selection mechanism~(\cref{sec:trajectory_selection}). 
For the latter, we also compare against selecting randomly, referred to as DPCC-R.
The results, including standard deviations, are shown in~\cref{fig:selection_tightening_impact}. 
Projecting onto the non-tightened state constraints~$\mathcal{S}_t$ only results in constraint satisfaction success rates below~$50\%$.
By utilizing our constraint-tightening approach, we achieve close to~$100\%$ success rate for all three trajectory selection criteria.
The main performance difference between the selection criteria is the number of time steps needed to reach the goal, which is higher for DPCC-R than for DPCC-T and DPCC-C.
One reason for this is that DPCC-T and, to a lesser extent, DPCC-C result in more temporally consistent closed-loop behavior, as shown in~\cref{fig:selection_criteria_visualization}.
These results show the potential of exploiting the fact that diffusion models can generate multiple trajectories at each timestep without increased computation.
\begin{figure}[tbh!]
    \small
    \centering
    \begin{minipage}{0.32\linewidth}
        \centering
        DPCC-R (random)\\
        \includegraphics[width=\linewidth, trim={0 0 0 0.1cm},clip]{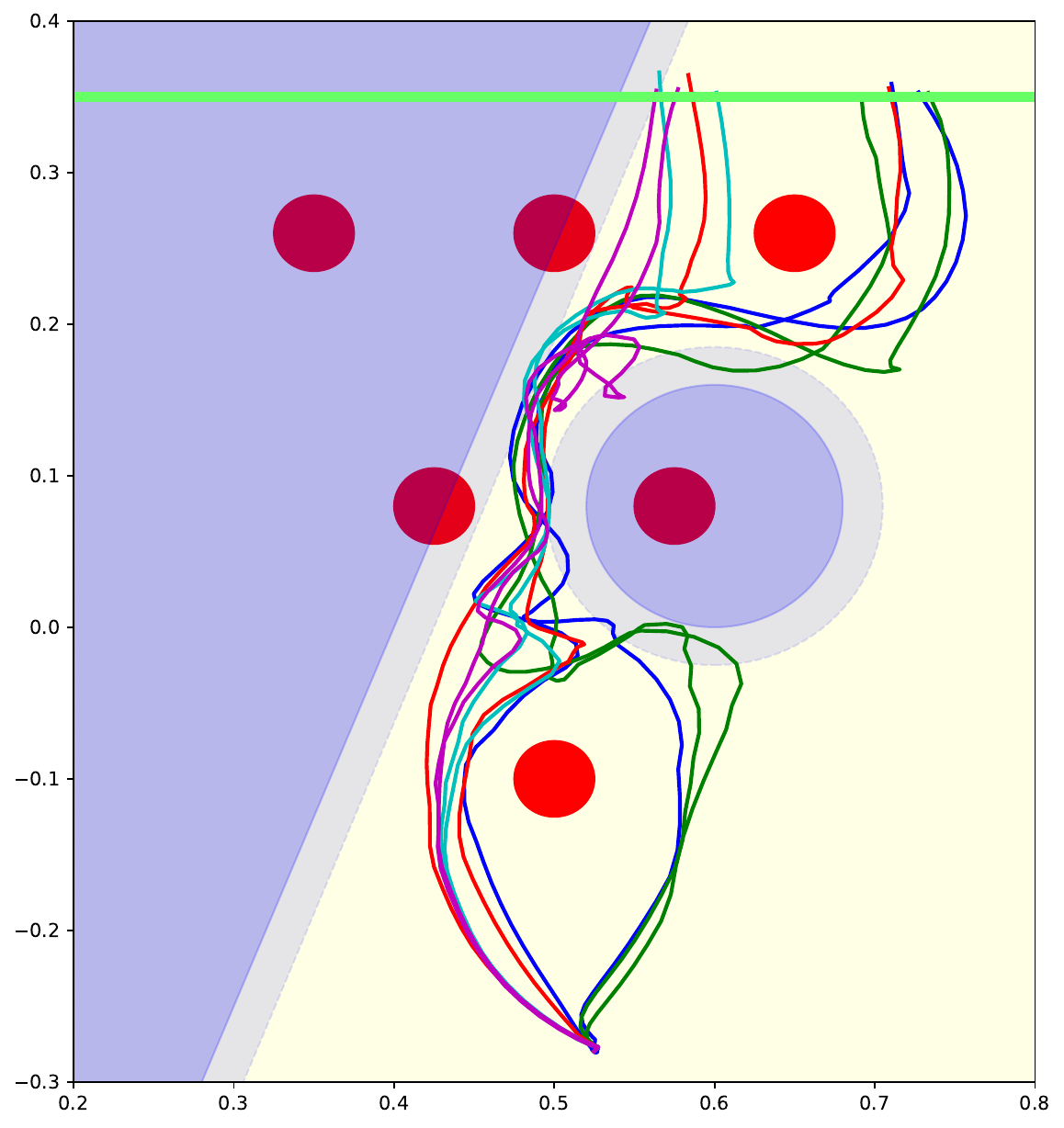}
    \end{minipage}
    \hfill
    \begin{minipage}{0.32\linewidth}
        \centering
        DPCC-T (temporal consistency)\\
        \includegraphics[width=\linewidth, trim={0 0 0 0.1cm},clip]{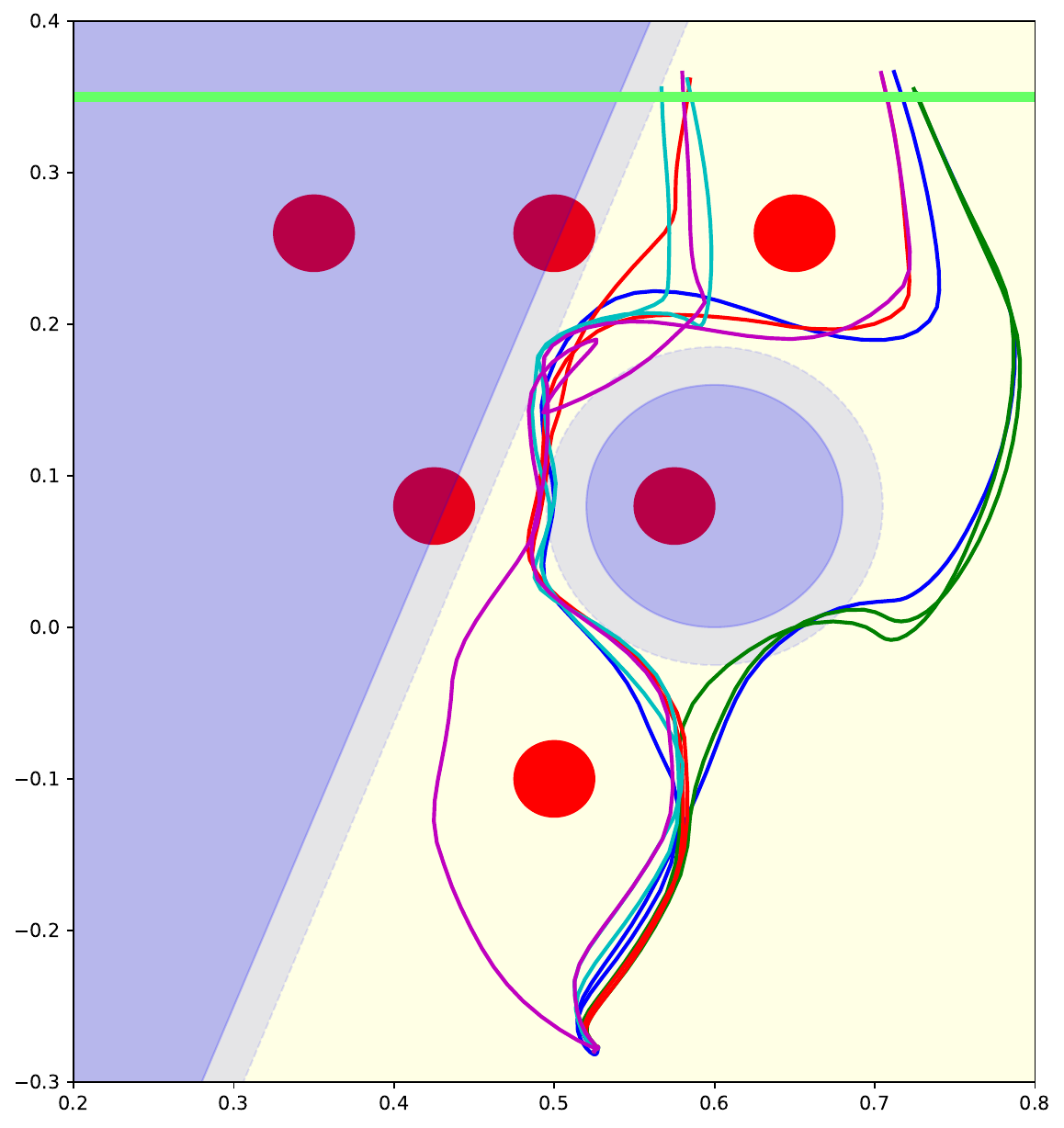}
    \end{minipage}
    \hfill
    \begin{minipage}{0.32\linewidth}
        \centering
        DPCC-C (projection costs)\\
        \includegraphics[width=\linewidth, trim={0 0 0 0.1cm},clip]{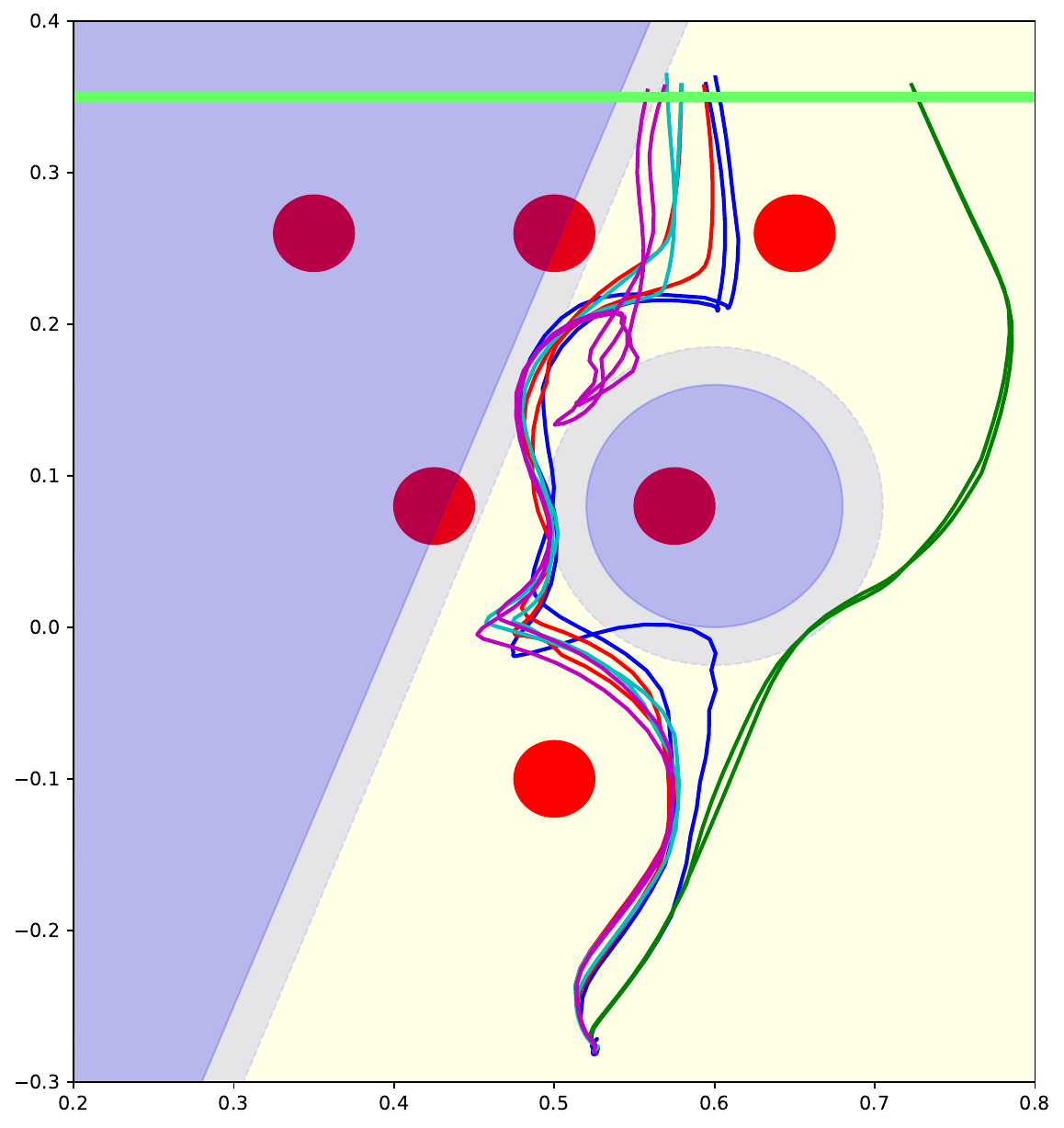}
    \end{minipage}
    \vspace{-6pt}
    \caption{Closed-loop trajectories with DPCC for different trajectory selection criteria and five training seeds, which are indicated by the trajectories' colors. The tightened constraints are visualized in light blue. With our proposed trajectory selection criteria (DPCC-T and DPCC-C), we obtain a smoother behavior and shorter time to reach the goal.}
    \label{fig:selection_criteria_visualization}
\end{figure}

We report the results of our baseline comparison (\textbf{Q2}) in~\cref{tab:constraints_performance} and also include Diffuser~\citep{janner22planning}, which does not take into account constraints at all.
DPCC-C has the highest success rate and the smallest number of constraint violations and, on average, reaches the goal significantly faster than all other methods.
This highlights that DPCC retains very good task performance by sampling approximately from the conditional distribution~\eqref{eq:meth_trajectory_distribution_feasible}, which encodes both the learned goal-reaching behavior and constraint satisfaction.
We find that for cost guidance, the trajectory is often either not pushed out of the unfeasible region completely, resulting in poor constraint performance, or pushed far away from the boundary of the feasible set, such that reaching the goal requires more timesteps.
The model-free projections perform poorly in our experiments.
This shows that although the learned distribution~\eqref{eq:meth_learned_traj_dist} contains information about the dynamics, the learned mean~$\bm{\mu}_{\bm{\theta}}$ cannot restore a denoised trajectory's dynamic feasibility if it is destroyed by iterative model-free projections, resulting in a trajectory that the system cannot follow.
Post-Processing performs better than the other baselines but also needs significantly more timesteps than DPCC-C.
\begin{table}[bt!]
    \scriptsize
    \centering
    \begin{tabular}{c c c c c c}
        \hline
        & Constraint tightening & Timesteps & Goal & Constraints \& goal & \# Constraint violations \\
        \hline        
        Diffuser 
        & - & $76.7^{\pm 12.7}$ & $\bm{1.00}$ & $0.07$ & $17.8^{\pm 12.1}$ \\
        \multirow{2}{*}{Guidance} & no
        & $74.5^{\pm 11.8}$ & $0.99$ & $0.09$ & $17.3^{\pm 12.6}$ \\
        & yes 
        & $75.6^{\pm 12.4}$ & $0.96$ & $0.13$ & $17.4^{\pm 14.2}$ \\
        \multirow{2}{*}{Post-Processing} & no
        & $84.5^{\pm 21.5}$ & $0.95$ & $0.32$ & $8.2^{\pm 13.9}$ \\ 
        & yes 
        & $79.1^{\pm 13.9}$ & $\bm{1.00}$ & $0.96$ & $0.1^{\pm 0.5}$ \\ 
        \multirow{2}{*}{Model-Free} & no
        & $76.7^{\pm 12.4}$ & $0.99$ & $0.07$ & $17.8^{\pm 12.4}$ \\
        & yes 
        & $76.1^{\pm 12.1}$ & $0.99$ & $0.07$ & $18.0^{\pm 12.0}$ \\
        \multirow{2}{*}{DPCC-C (ours)} & no
        & $76.1^{\pm 21.5}$ & $0.96$ & $0.49$ & $6.0^{\pm 11.5}$ \\
        & yes 
        & $\bm{69.0^{\pm 12.9}}$ & $\bm{1.00}$ & $\bm{0.98}$ & $\bm{0.0^{\pm 0.3}}$ \\
        \hline
    \end{tabular}
    \vspace{-6pt}
    \caption{Comparison of DPCC against other approaches for diffusion-based receding-horizon control with constraints.}
    \label{tab:constraints_performance}
    \vspace{-10pt}
\end{table}

\setlength{\tabcolsep}{3.5pt}
\begin{wraptable}{r}{0.56\textwidth}
    \vspace{-4pt}
    \scriptsize
    \centering
    \begin{tabular}{c c c c c}
        \hline
        $\hat{t}_{\text{s}} / t_{\text{s}}$ & Timesteps & Goal & Constraints \& goal & \# Constraint violations \\
        \hline
        $0.25$ & $85.7^{\pm16.4}$ & $\bm{1.00}$ & $0.86$ & $0.3^{\pm0.8}$ \\
        $0.5$ & $73.1^{\pm9.8}$ & $\bm{1.00}$ & $\bm{0.99}$ & $\bm{0.0^{\pm0.3}}$ \\
        $1$ & $\bm{69.0^{\pm 12.9}}$ & $\bm{1.00}$ & $0.98$ & $\bm{0.0^{\pm0.3}}$ \\
        $2$ & $76.6^{\pm14.8}$ & $0.99$ & $0.95$ & $0.3^{\pm2.1}$ \\
        $4$ & $152.0^{\pm26.3}$ & $0.88$ & $0.77$ & $0.6^{\pm 1.8}$ \\
        \hline
    \end{tabular}
    \vspace{-5pt}
    \caption{Impact of the model mismatch between the dynamics used in the constraint set projection~\eqref{eq:meth_projection_with_prior}, which assume a sampling time~$\hat{t}_{\text{s}}$, and the true dynamics with sampling time~$t_{\text{s}}$, for DPCC-C.}
    \label{tab:model_mismatch_impact}
\end{wraptable}

We have seen that neglecting the dynamics in the projections results in poor constraint satisfaction.
To better understand the impact of the dynamics model used in~\eqref{eq:meth_projection_with_prior} (\textbf{Q3}), we consider a mismatch between the assumed sampling time~$\hat{t}_{\text{s}}$ and its true value~$t_{\text{s}}$.
The results provided in~\cref{tab:model_mismatch_impact} demonstrate that even with a significant deviation by a factor of~$4$, the constraints can be satisfied in most cases. 
This shows that the iterative constraint set projections yield much better results with even a very inaccurate dynamics model than when using no model.

\vspace{-5pt}
\section{Conclusion}
\vspace{-2pt}
DPCC combines the expressivity of diffusion models for offline policy learning with the ability of predictive control to satisfy constraints online in closed-loop operation.
We show that incorporating model-based projections into the trajectory denoising process allows us to sample future trajectories that are constraint-satisfying, dynamically feasible, and suitable for solving the learned task.
Our experiments do not consider time-varying constraints, but DPCC can handle them directly without modifications.
In future work, we aim to include additional notions of safety, such as stability.

\newpage
\section*{Acknowledgements}
Ralf R\"omer gratefully acknowledges the support of the research group ConVeY funded by the German Research Foundation under grant GRK 2428. This work has been supported by the Robotics Institute Germany, funded by BMBF grant 16ME0997K.
\bibliography{ref}

\newpage
\appendix
\label{sec:appendix}
\section{Training Details}
Our training hyperparameters are provided in~\cref{tab:hyperparams}.
We use the cosine learning rate scheduler~\citep{von_platen2024diffusers} with~$10^3$ warmup steps.
We split the dataset into~$90\%$ training and~$10\%$ validation data and use the model that performs best on the validation data for testing.
Training a single trajectory diffusion model takes about~$30$~minutes on a workstation with 64 GB RAM, an NVIDIA Geforce RTX 4090 GPU, and an Intel Core i7-12800HX CPU.
\begin{table}[h!]
    \small
    \centering
    \begin{tabular}{c c}
    \hline
        Hyperparameter & Value \\
        \hline
        Optimizer & Adam~\citep{kingma2014adam} \\
        Batch size & $8$ \\
        Diffusion steps~$K$ & $20$ \\
        Learning rate & $1 \times 10^{-4}$ \\
        Training steps & $10^5$ \\
        Epochs & $100$ \\
        \hline
    \end{tabular}
    \caption{Hyperparameters for training the trajectory diffusion model.}
    \label{tab:hyperparams}
\end{table}

\section{Testing Details}
The novel state constraints visualized in~\cref{fig:avoiding}~(c) are satisfied by~$10.4\%$, $4.2\%$ and~$7.3\%$ of the demonstrations shown in~\cref{fig:avoiding}~(b), and the tightened constraints are satisfied by~$1.0\%$, $0\%$ and~$2.1\%$ of the demonstrations, respectively.
Therefore, we cannot use rejection sampling from the learned distribution to obtain constraint-satisfying trajectories.
The state constraint sets can be formulated as
\begin{align*}
    \mathcal{S}_t = \big\{\bs_t \in \mathcal{S}|\,\bm{A} \bs_t \leq \bm{b}, \;\|\bs_t - \bm{p}\|_2^2 \geq r^2\big\}, \qquad \forall t.
\end{align*}
We assume that the actual and desired position evolve via the Euler integration
\begin{align*}
    {\bs_{t+1} = \bm{s}_t + 
    \begin{bmatrix}
        \ba_t \\ \ba_t
    \end{bmatrix} 
    t_{\text{s}} + \bm{w}_t},
\end{align*}
which results in additional linear equality constraints on the generated trajectories.
As~$\mathcal{S}_t$ is non-convex, computing the trajectory projection into the set~$\tilde{\mathcal{Z}}_{\bm{f}}$ requires solving a non-convex quadratic optimization problem.
For this, we use a nonlinear SLSQP solver~\citep{virtanen2020scipy, kraft1994algorithm}.
Because the trajectories generated by the diffusion model are normalized, we normalize the constraint sets as well to compute the projections more efficiently without having to un-normalize the trajectories.
For this, we use the fact that limit normalization via
\begin{align*}
    s_{n,i} = 2\frac{s_i - \underline{s}_i}{\bar{s}_i - \underline{s}_i} - 1,
\end{align*}
where~$s_{n,i}$ is the normalized state dimension~$i$, and~$\bar{s}_i$ and~$\underline{s}_i$ are the corresponding upper and lower limit of the states in the training dataset, respectively, is an affine transformation.



\end{document}